\documentclass[runningheads]{llncs}
\usepackage[T1]{fontenc}
\usepackage{graphicx,verbatim}
\usepackage{amsmath,amssymb}
\usepackage{booktabs}
\usepackage{algorithm}
\usepackage{algorithmic}
\usepackage[hidelinks]{hyperref}

\begin{document}

\title{When Average Calibration Fails:\\
Site-Conditional Federated\\
Conformal Risk Control}
\titlerunning{When Average Calibration Fails: Site-Conditional Federated CRC}

\author{Nafis Fuad Shahid}
\authorrunning{N.F. Shahid}
\institute{Dhaka, Bangladesh \\
    \email{nafisfuadshahid@gmail.com}}

\maketitle

\begin{abstract}
Conformal risk control (CRC) provides distribution-free segmentation guarantees by calibrating a prediction-set threshold on held-out data.
In federated deployments, the standard approach pools calibration scores into a single threshold.
We quantify, on real multi-institutional brain tumor data (FeTS-2022, 1{,}251 subjects, 20~institutions), a critical failure: \emph{naive pooled} CRC protects the average hospital but violates coverage at 40\% of individual institutions, with the worst site exceeding the target false-negative rate by 7.8 percentage points.
We trace this failure to a hidden design choice: the aggregation weights implicitly determine \emph{whose} coverage is protected.
Sample-size weighting optimizes patient-level validity but can sacrifice institution-level reliability; equal-site weighting improves institution-level reliability on this benchmark at comparable efficiency, using only a single scalar per site.
We propose risk-curve shrinkage as a principled mechanism: each site transmits its empirical risk curve ($G$ scalars) and a single hyperparameter $n_0$ smoothly interpolates between site-specific local calibration and sample-size-weighted pooled calibration.
Leave-one-site-out sensitivity analysis identifies $n_0{=}19$, achieving 2.7/20 violations at $2.0\times$ stretch.
Direct Lagrangian budget optimization \emph{fails} by concentrating risk on vulnerable hospitals; the finite-sample correction term is essential: removing it triples violations.
No patient-level images, masks, or per-volume scores leave any site. 
Code implementation is available at: \url{https://github.com/NafisFuadShahid/Fed-CRC-Seg}.

\end{abstract}

\section{Introduction}
\label{sec:intro}

Deploying segmentation models across hospitals requires calibrated uncertainty: clinicians must know when a model's output can be trusted and when it cannot.
Conformal risk control (CRC)~\cite{angelopoulos2024crc} provides distribution-free, finite-sample guarantees of the form $\mathbb{E}[\ell(C_\lambda(X), Y)] \leq \alpha$, where $\ell$ is a monotone loss and $\lambda$ is calibrated on held-out data.
CRC has been applied to centralized medical segmentation, including morphological dilation families~\cite{mossina2025}, conditional CRC~\cite{luo2025cra}, and semantic CRC~\cite{teneggi2025}, but never in the federated setting where calibration data is distributed across heterogeneous institutions that cannot share patient data.

In federated learning (FL)~\cite{mcmahan2017}, the natural approach pools calibration scores from all sites to compute a single global threshold, inheriting federated conformal prediction machinery~\cite{lu2023fedcp}.
We show this \emph{naive pooled CRC} harbors a critical failure mode: it satisfies the marginal coverage guarantee while violating coverage at 40\% of individual institutions.
On FeTS-2022 brain tumor data~\cite{fets2022} with 20 institutions, the worst site reaches FNR${=}0.178$; such site-specific errors could translate into missed tumor regions in downstream clinical workflows, motivating institution-level evaluation.
This marginal-vs-conditional gap is well studied in CP theory~\cite{vovk2012,barber2021limits}, but its magnitude in federated medical segmentation has not been quantified.

Crucially, we show this failure is not inherent to scalar communication: an \emph{unweighted} scalar-threshold average achieves comparable reliability to our full method.
The root cause is the \emph{aggregation weighting}: sample-size weights optimize patient-level validity but can sacrifice institution-level reliability.
This reveals federated calibration as an \emph{objective-selection} problem, not merely an estimation problem.

\textbf{Contributions.}
(1)~We quantify the marginal-conditional coverage gap for federated CRC on 1{,}251 real multi-institutional brain tumor volumes: 8/20 institutions fail while the average appears calibrated.
(2)~We identify aggregation weighting as a hidden determinant of federated coverage: sample-size weighting targets patient-level validity, equal-site weighting targets institution-level reliability, and these diverge on FeTS-2022.
(3)~We propose risk-curve shrinkage as a controllable mechanism, with $n_0$ interpolating between local and pooled calibration, validated via sensitivity analysis.
(4)~Direct budget optimization fails; the finite-sample correction is essential: removing it triples violations.

\section{Related Work}
\label{sec:related}

\textbf{CRC for segmentation.}
Angelopoulos et al.~\cite{angelopoulos2024crc} introduced CRC for monotone losses.
Recent centralized extensions include morphological dilation families~\cite{mossina2025}, conditional CRC~\cite{luo2025cra}, and semantic CRC~\cite{teneggi2025}.
All assume centralized calibration data.

\textbf{Federated conformal prediction.}
Lu et al.~\cite{lu2023fedcp} prove federated CP coverage under partial exchangeability for classification.
Extensions address group-conditional~\cite{gcfcp2026}, Byzantine-robust~\cite{kang2024}, privacy-preserving~\cite{plassier2023fedcp}, and weighted-quantile~\cite{fedwqcp2026} settings.
To our knowledge, these methods do not address pixel-level segmentation CRC.
The cell \emph{federated $\times$ CRC $\times$ pixel-level segmentation} is empty.
Exact conditional coverage is impossible without structural assumptions~\cite{vovk2012,barber2021limits}; our per-site coverage is empirical, viewed as approximate group-conditional coverage where each ``group'' is a hospital.

\textbf{Federated segmentation with uncertainty.}
FUNAvg~\cite{funavg2024} aggregates MC-dropout uncertainty; FedEvi~\cite{fedevi2024} uses evidential learning.
Neither provides formal coverage guarantees.
FedStein~\cite{fedstein2025} applies James--Stein estimation to federated batch-normalization statistics; we apply shrinkage to empirical risk curves for post-hoc calibration.

\section{Method}
\label{sec:method}

\subsection{Setting}

Consider $K$ sites with $n_k$ calibration volumes $(X_i^k, Y_i^k)$ each.
A pre-trained model $f$ produces predicted probabilities via sigmoid activation.
We define nested prediction sets $C_\lambda(X) = \{v : f(X)_v \geq 1 - \lambda\}$, growing with $\lambda$.
The per-volume false-negative rate $\ell(C_\lambda, Y) = 1 - |C_\lambda \cap Y|/|Y|$ is non-increasing in $\lambda$ and bounded in $[0, 1]$, so the CRC loss bound $B = \sup \ell = 1$~\cite{angelopoulos2024crc}.
Goal: find $\hat\lambda$ such that $\mathbb{E}[\ell(C_{\hat\lambda}(X_\mathrm{test}), Y_\mathrm{test})] \leq \alpha$.

\subsection{Pooled CRC and Its Hidden Objective}

Centralized CRC~\cite{angelopoulos2024crc} sets $\hat\lambda = \inf\{\lambda : \hat{R}(\lambda) + B/(n{+}1) \leq \alpha\}$ where $\hat{R}(\lambda) = \frac{1}{n}\sum_i \ell_i(\lambda)$.
\emph{Naive pooled federated CRC} pools all $N = \sum_k n_k$ scores, implicitly optimizing the \textbf{patient-uniform} objective $R_\mathrm{micro}(\lambda) = \sum_k (n_k/N) R_k(\lambda)$: large hospitals dominate.
Under partial exchangeability~\cite{lu2023fedcp} this controls marginal risk but \emph{not} site-conditional risk $\mathbb{E}[\ell \mid H{=}k]$.
\\
An alternative \textbf{site-uniform} objective treats each institution equally: $R_\mathrm{macro}(\lambda) = \frac{1}{K}\sum_k R_k(\lambda)$.
The gap between these objectives is:
\begin{equation}
R_\mathrm{micro}(\lambda) - R_\mathrm{macro}(\lambda) = \frac{\mathrm{Cov}_k(n_k, R_k(\lambda))}{\bar{n}},
\label{eq:gap}
\end{equation}
which is nonzero whenever hospital size correlates with site risk.
On FeTS-2022, the contrast between weighted and unweighted threshold aggregation in Table~\ref{tab:main} shows that the weights assigned to local operating points materially affect site-level reliability.

\subsection{Calibration Poverty at Small Sites}

Each site can independently set $\hat\lambda_k = \inf\{\lambda : \hat{R}_k(\lambda) + B/(n_k{+}1) \leq \alpha\}$.
This guarantees site-conditional coverage, but for small $n_k$ the finite-sample correction $B/(n_k{+}1)$ is prohibitively large: with $n_k{=}5$ and $B{=}1$ it alone exceeds $\alpha{=}0.10$, making nontrivial calibration impossible and forcing $\hat\lambda_k = \lambda_\mathrm{max}$ with stretch above $80\times$.
We term this \emph{calibration poverty}: small hospitals cannot calibrate locally without enormous prediction sets, motivating information borrowing from the federation.

\subsection{Shrinkage-Based Federated CRC}

We propose an approach that borrows information at the \emph{risk-curve} level (Fig.~\ref{fig:overview}).
Each site transmits its empirical risk curve $\hat{R}_k(\lambda)$ on a grid $\Lambda$ of $G$ points.
The server computes $\hat{R}_\mathrm{global}(\lambda) = \frac{1}{N}\sum_k n_k \hat{R}_k(\lambda)$ and the \textbf{shrinkage risk curve}:
\begin{equation}
\hat{R}_k^\mathrm{shrink}(\lambda) = w_k \cdot \hat{R}_k(\lambda) + (1 - w_k) \cdot \hat{R}_\mathrm{global}(\lambda) + \mathrm{corr}_k,
\label{eq:shrinkage}
\end{equation}
where $w_k = n_k/(n_k + n_0)$ is the shrinkage weight, the standard empirical Bayes form where $n_0$ acts as prior precision~\cite{fedstein2025}, and $\mathrm{corr}_k = w_k \cdot B/(n_k{+}1) + (1{-}w_k) \cdot B/(N{+}1)$ is a heuristic interpolation between per-site and pooled finite-sample corrections, motivated by but not formally derived from CRC theory.
The threshold is $\hat\lambda_k^\mathrm{shrink} = \inf\{\lambda \in \Lambda : \hat{R}_k^\mathrm{shrink}(\lambda) \leq \alpha\}$.
As $n_0{\to}0$ this recovers per-site local CRC; as $n_0{\to}\infty$ it recovers pooled CRC.
The hyperparameter $n_0$ thus provides a dial between local, site-specific calibration and pooled, patient-weighted calibration.

\textbf{Privacy.} Each site transmits $G$ real numbers summarizing aggregate loss statistics.
No individual volumes, masks, or per-volume scores leave the site.
Total bandwidth: ${\sim}0.8$\,KB for $G{=}200$.

\begin{figure}[t]
\centering
\includegraphics[width=\textwidth]{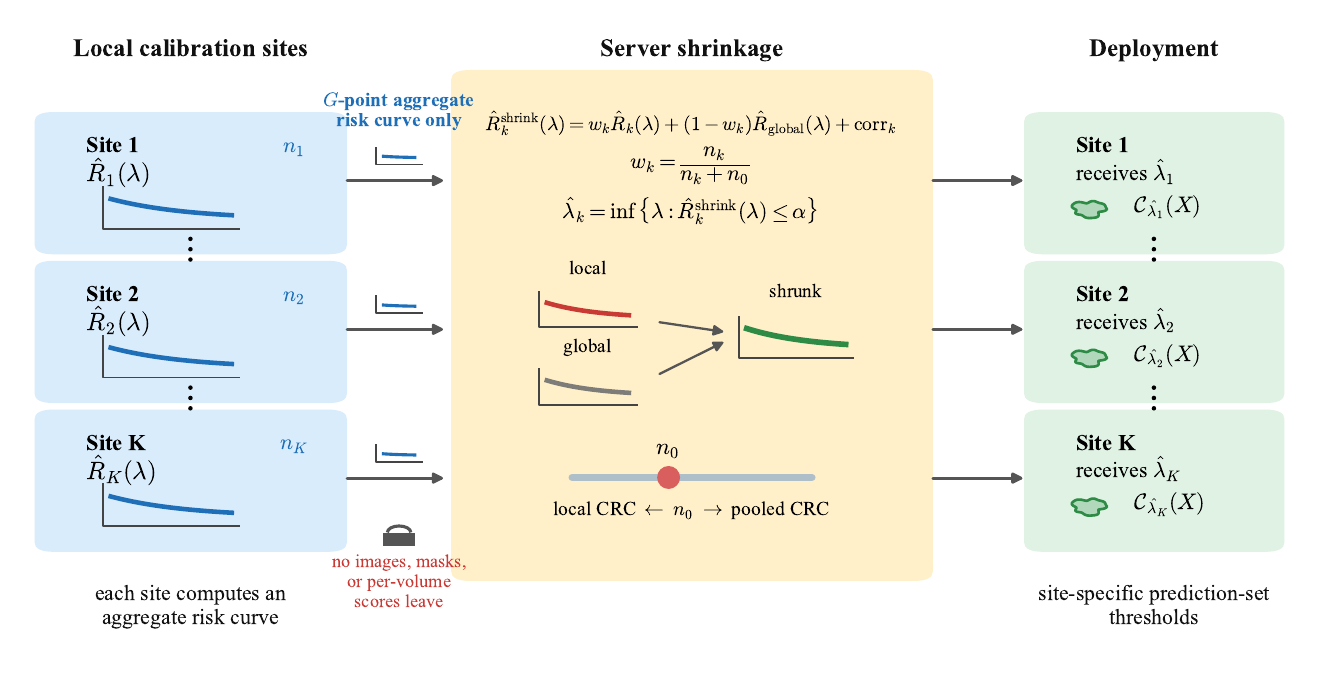}
\caption{Protocol overview. Each site transmits only its empirical risk curve ($G$ scalars, ${\sim}0.8$\,KB for $G{=}200$) to the server, which computes shrinkage-regularized per-site thresholds and broadcasts them back. No patient-level images, masks, or per-volume scores leave any site.}\label{fig:overview}
\end{figure}

\begin{algorithm}[t]
\caption{Shrinkage-Based Federated CRC}\label{alg:fedcrc}
\begin{algorithmic}[1]
\REQUIRE Sites $k \in [K]$; target $\alpha$; grid $\Lambda$; prior strength $n_0$
\STATE Each site $k$: compute $\hat{R}_k(\lambda)$ for $\lambda \in \Lambda$; enforce monotonicity via cumulative minimum over increasing $\lambda$; send to server
\STATE Server: $\hat{R}_\mathrm{global}(\lambda) \leftarrow \frac{1}{N}\sum_k n_k \hat{R}_k(\lambda)$
\FOR{each site $k$}
  \STATE Compute $\hat{R}_k^\mathrm{shrink}$ (Eq.~\eqref{eq:shrinkage}); $\hat\lambda_k \leftarrow \inf\{\lambda : \hat{R}_k^\mathrm{shrink}(\lambda) \leq \alpha\}$
\ENDFOR
\STATE Broadcast $\hat\lambda_k$ to site $k$
\end{algorithmic}
\end{algorithm}

\subsection{Mechanistic Illustration on Real Calibration Curves}
\label{sec:illustration}

Fig.~\ref{fig:riskgeo} shows representative empirical risk upper bounds
$\hat{R}_k(\lambda) + \mathrm{corr}_k$ from one calibration split (seed 42).
The pooled curve crosses $\alpha{=}0.10$ at $\hat\lambda_\mathrm{pool}{\approx}0.95$,
but the hard site (inst.~4, $n_k{=}23$) remains well above $\alpha$ at this
threshold, explaining the coverage violation reported in Table~\ref{tab:main}.
The shrinkage curve for inst.~4 ($n_0{=}19$) interpolates between pooled and
local, crossing $\alpha$ at a practical $\hat\lambda_\mathrm{ours}$.
The easy site (inst.~7, $n_k{=}6$) never crosses $\alpha$ within the plotted
range; its local threshold falls near $\lambda{=}1$ (predicting the entire
volume as tumor), illustrating the extreme stretch incurred by per-site CRC
at small sites.

\begin{figure}[t]
\centering
\includegraphics[width=\textwidth]{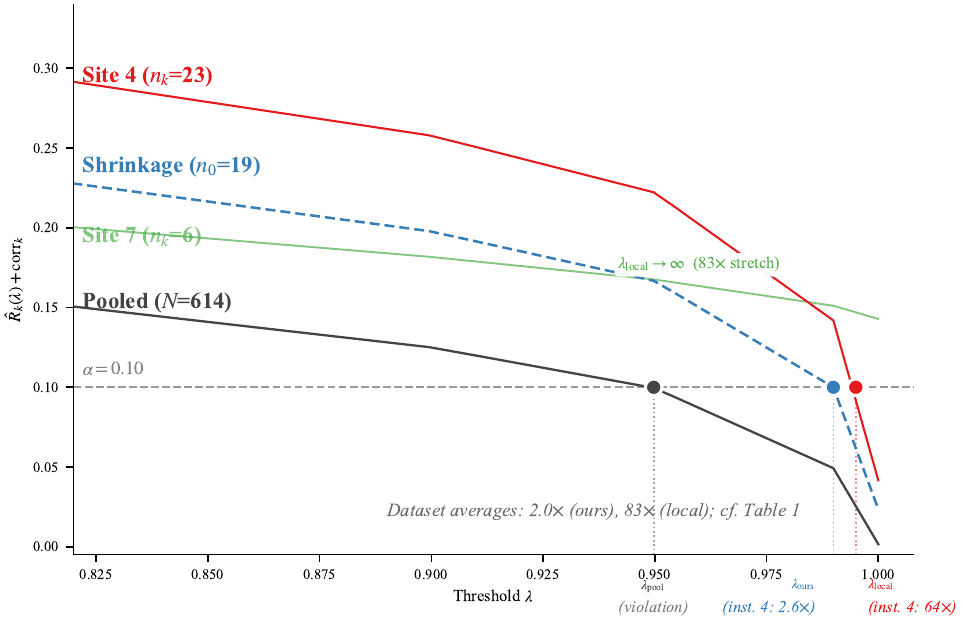}
\caption{Representative empirical risk upper bounds $\hat{R}_k(\lambda) +
\mathrm{corr}_k$ from one calibration split (seed 42), showing two
FeTS-2022 institutions: easy (inst.~7, $n_k{=}6$) and hard (inst.~4,
$n_k{=}23$). At $\hat\lambda_\mathrm{pool}$, the hard site's risk bound
exceeds $\alpha{=}0.10$, the failure mode our method addresses.
Shrinkage ($n_0{=}19$) selects $\hat\lambda_\mathrm{ours}$ between
$\hat\lambda_\mathrm{pool}$ and $\hat\lambda_\mathrm{local}$, achieving
near-nominal coverage without extreme stretch.}\label{fig:riskgeo}
\end{figure}

\subsection{Marginal Safeguard}

\begin{proposition}[Marginal monotonicity safeguard]\label{prop:marginal}
Assume that the pooled threshold $\hat\lambda_\mathrm{pool}$ satisfies
$\mathbb{E}[\ell(C_{\hat\lambda_\mathrm{pool}}(X_\mathrm{test}), Y_\mathrm{test})] \leq \alpha$
under the target test mixture.
For arbitrary site-specific thresholds $\{\hat\lambda_k\}$, define
$\hat\lambda_\mathrm{fed} = \max\bigl\{\hat\lambda_\mathrm{pool},\, \max_k \hat\lambda_k\bigr\}$.
Then
$\mathbb{E}[\ell(C_{\hat\lambda_\mathrm{fed}}(X_\mathrm{test}), Y_\mathrm{test})] \leq \alpha$.
\end{proposition}

\begin{proof}
$\hat\lambda_\mathrm{fed} \geq \hat\lambda_\mathrm{pool}$ and $\ell$ non-increasing in $\lambda$ give $\ell(C_{\hat\lambda_\mathrm{fed}}) \leq \ell(C_{\hat\lambda_\mathrm{pool}})$ pointwise, so the expectation bound follows.
\end{proof}

\textbf{Remark.}
This safeguard holds for \emph{any} choice of per-site thresholds, as long as the deployed threshold includes $\hat\lambda_\mathrm{pool}$ in the maximum.
Deploying $\hat\lambda_\mathrm{fed}$ globally achieves 0~violations at $n_0{=}9$ but at $67\times$ stretch (Table~\ref{tab:ablation}); per-site thresholds are therefore used in practice and evaluated empirically.
Exact conditional coverage is impossible without structural assumptions~\cite{vovk2012,barber2021limits}.

\section{Experiments}
\label{sec:experiments}

\subsection{Setup}

\textbf{Data.}
FeTS-2022 training set~\cite{fets2022}: 1{,}251 multi-modal brain MRI volumes from 23~institutions.
After excluding institutions with fewer than six subjects, our analysis retains 20~institutions (calibration sizes: 3--255).
Per-site 50/50 cal/test splits with seeds $\{42, 1337, 2024\}$.

\textbf{Model.}
Pre-trained SegResNet from the MONAI model zoo~\cite{monai}, trained on BraTS-2021.

\textbf{Baselines.}
\textbf{B3 (Naive Pooled):} pools all scores (patient-uniform weighting).
\textbf{B2 (Per-site Local):} independent per-site thresholds.
To test whether the failure is caused by scalar communication or by weighting, we include two scalar-threshold baselines: each site computes its local CRC threshold $\hat\lambda_k$ and the server broadcasts either the \textbf{sample-size-weighted average}~$\hat\lambda = \sum_k (n_k/N)\hat\lambda_k$, following FedWQ-CP~\cite{fedwqcp2026}, or the \textbf{equal-site average}~$\hat\lambda = \frac{1}{K}\sum_k \hat\lambda_k$.

\textbf{CRC details.}
Loss: per-volume pixel-FNR.
Prediction sets: $C_\lambda = \{v : f(X)_v \geq 1{-}\lambda\}$ on $G{=}200$ uniformly spaced $\lambda \in [0,1]$, monotonicity enforced via cumulative minimum over increasing $\lambda$.
\\
Target $\alpha = 0.10$, swept over $\{0.05, 0.10, 0.15, 0.20\}$.

\textbf{Metrics.}
(i)~Violations: sites with mean test FNR $> \alpha$ (mean$\pm$std over seeds).
(ii)~Worst-site FNR (mean over seeds).
(iii)~Stretch: $|C_{\hat\lambda}|/|Y|$ (mean over seeds).

\subsection{Main Results}

Table~\ref{tab:main} presents results across three seeds.
B3 violates coverage at $8.0 \pm 2.4$ of 20~sites, with worst FNR $= 0.178$, nearly double the target.
B2 reduces violations to~$1.3$ but inflates stretch to $83\times$ due to calibration poverty at small sites.

\textbf{The weighting, not the representation, drives the failure.}
Weighted scalar aggregation ($8.3 \pm 2.9$ violations) performs comparably to naive pooling: finite-sample correction pushes many calibration-poor sites to the conservative boundary $\hat\lambda_k{=}\lambda_\mathrm{max}$, and sample-size weighting suppresses their influence.
Unweighted scalar aggregation preserves the influence of these conservative local thresholds, achieving $1.7 \pm 0.9$ violations at $2.3\times$ stretch.
Both scalar baselines require only 4~bytes/site, compared with ${\sim}800$~bytes/site for risk-curve transmission.

Risk-curve shrinkage provides a \emph{controllable} trade-off.
At $n_0{=}9$: 1.3~violations, $29\times$ stretch.
At $n_0{=}15$: 2.7~violations, $4.4\times$ stretch.
LOSO sensitivity analysis (Sec.~\ref{sec:n0}) identifies $n_0{=}19$: 2.7~violations at $2.0\times$ stretch, a 97.6\% reduction in average stretch relative to B2.
The failure is not a small-sample artifact: Fig.~\ref{fig:money} shows that Institution~18 (382~patients, 30\% of the dataset) is miscovered at FNR${=}0.150$ under pooling, and B2's stretch explodes to $100{-}250\times$ at small sites, while ours reduces average stretch to $4.4\times$.

\begin{table}[t]
\caption{Results on FeTS-2022 (20 institutions, $\alpha{=}0.10$, 3 seeds). Violations: mean$\pm$std; worst FNR and stretch: mean over seeds.}\label{tab:main}
\centering\small
\begin{tabular}{lccc}
\toprule
\textbf{Method} & \textbf{Violations} ($\downarrow$) & \textbf{Worst FNR} ($\downarrow$) & \textbf{Stretch} ($\downarrow$) \\
\midrule
B3: Naive Pooled & $8.0 \pm 2.4$ & $0.178$ & $1.5\times$ \\
B2: Per-site Local & $1.3 \pm 1.2$ & $0.111$ & $83.2\times$ \\
\midrule
Weighted $\lambda$ Agg. & $8.3 \pm 2.9$ & $0.186$ & $1.4\times$ \\
Unweighted $\lambda$ Agg. & $1.7 \pm 0.9$ & $0.113$ & $2.3\times$ \\
\midrule
Budget Alloc.\ (uncapped) & $12.3 \pm 0.9$ & $0.351$ & $1.4\times$ \\
Budget Alloc.\ ($\delta{=}0.03$) & $2.7 \pm 1.2$ & $0.142$ & $71.6\times$ \\
\midrule
Ours ($n_0{=}9$) & $1.3 \pm 1.2$ & $0.112$ & $28.8\times$ \\
Ours ($n_0{=}15$) & $2.7 \pm 1.7$ & $0.119$ & $4.4\times$ \\
Ours ($n_0{=}19$, LOSO) & $2.7 \pm 1.7$ & $0.125$ & $\mathbf{2.0\times}$ \\
\bottomrule
\end{tabular}
\end{table}

\begin{figure}[t]
\centering
\includegraphics[width=\textwidth]{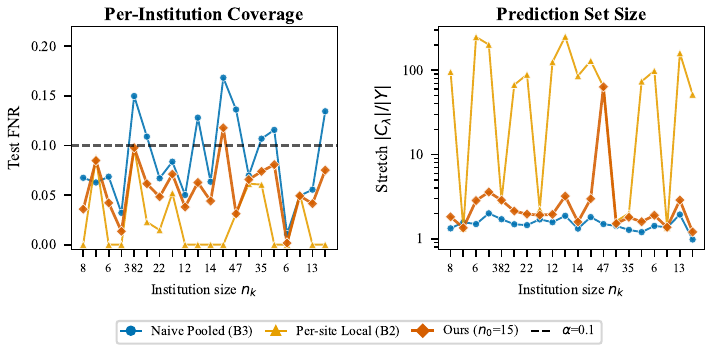}
\caption{Per-institution FNR (left) and prediction-set stretch (right) on FeTS-2022. B3 (blue) violates $\alpha{=}0.10$ at 8/20 sites; B2 (orange) largely restores coverage at $83\times$ stretch. Ours (red, $n_0{=}15$): practical stretch with 2--3 borderline violations. Scalar-aggregation baselines are reported in Table~\ref{tab:main}.}\label{fig:money}
\end{figure}

\subsection{The $n_0$ Dial and LOSO Sensitivity Analysis}
\label{sec:n0}

Fig.~\ref{fig:n0} shows how $n_0$ controls the coverage-efficiency frontier.
Small $n_0$ ($\leq 9$): few violations, high stretch (local regime).
Large $n_0$ ($\geq 30$): low stretch, many violations (pooled regime).
The knee at $n_0 \in [10, 20]$ offers the best trade-off.

As a post-hoc sensitivity analysis, we perform leave-one-site-out evaluation: for each candidate $n_0$, hold out one site, recompute the global curve from the remaining $K{-}1$ sites, compute the shrinkage threshold for the held-out site using its own calibration curve, and evaluate on the held-out site's test partition.
LOSO identifies $n_0{=}19$ as the lowest-stretch operating point with $\leq$3 mean violations.
The failure mode persists across all tested $\alpha \in \{0.05, 0.10, 0.15, 0.20\}$: B3 violates 5--8 sites at each target while ours ($n_0{=}9$) reduces violations to 0--1.

\begin{figure}[t]
\centering
\includegraphics[width=\textwidth]{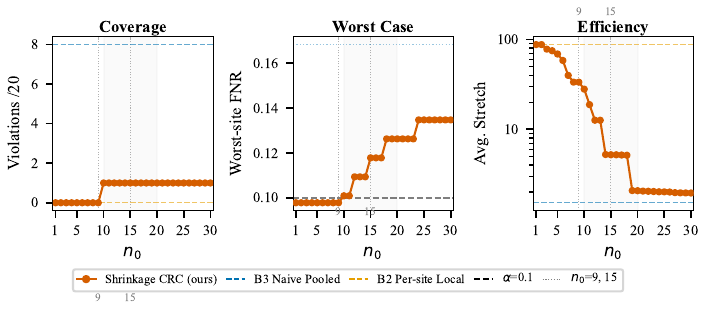}
\caption{Effect of $n_0$ on violations (left), worst-site FNR (center), and stretch (right). Dashed: B3 (blue) and B2 (orange). Shaded: sweet spot $n_0 \in [10, 20]$.}\label{fig:n0}
\end{figure}

\subsection{Why Direct Optimization Fails}

A budget-allocation baseline minimizes total stretch subject to marginal risk $\leq \alpha$: $\min_{\lambda_1,\ldots,\lambda_K} \sum_k S_k(\lambda_k)$ s.t.\ $\sum_k p_k L_k(\lambda_k) \leq \alpha$, where $L_k(\lambda) = \hat{R}_k(\lambda) + B/(n_k{+}1)$, $S_k$ is mean stretch at site $k$, and $p_k = (n_k{+}1)/(N{+}K)$.
We binary-search on the Lagrange multiplier $\mu$ until the constraint binds.

Uncapped, this achieves $1.4\times$ stretch but fails 12.3/20~sites (worst FNR${=}0.351$): the optimizer concentrates risk on small, hard institutions.
Adding a per-site cap ($\delta{=}0.03$) reduces violations to~2.7 but inflates stretch to $71.6\times$, worse than shrinkage on \emph{both} dimensions.

\subsection{Ablations}
\label{sec:ablations}

Table~\ref{tab:ablation} probes three axes.

\textbf{The correction $\mathrm{corr}_k$ is essential}, making this the most consequential ablation.
Removing $\mathrm{corr}_k$: violations jump from 1.3--2.7 to 8.0--9.3 across all $n_0$, comparable to naive pooling.
$\mathrm{corr}_k$ interpolates between local and pooled corrections; its formal analysis is future work, but its empirical necessity is unambiguous.

\textbf{Conservative $\hat\lambda_\mathrm{fed}$.}
The global threshold of Proposition~\ref{prop:marginal} achieves 0~violations at $n_0{=}9$ but $67\times$ stretch, confirming per-site deployment is necessary for clinical utility.

\textbf{Grid $G$.}
$G{=}200$ is stable; $G{=}50$ inflates stretch to $35.7\times$ (too coarse); $G{=}500$ yields $3.0\times$ with similar violations.

\begin{table}[t]
\caption{Ablation results ($\alpha{=}0.10$, mean over 3 seeds).}\label{tab:ablation}
\centering\small
\begin{tabular}{lcccc}
\toprule
\textbf{Ablation} & $n_0$ & \textbf{Viol.} & \textbf{W-FNR} & \textbf{Stretch} \\
\midrule
No $\mathrm{corr}_k$ & 9 & $9.3 \pm 2.6$ & $0.181$ & $1.5\times$ \\
No $\mathrm{corr}_k$ & 15 & $8.3 \pm 2.6$ & $0.173$ & $1.5\times$ \\
\midrule
$\hat\lambda_\mathrm{fed}$ (Prop.~\ref{prop:marginal}) & 9 & $0.0 \pm 0.0$ & $0.050$ & $67.3\times$ \\
$\hat\lambda_\mathrm{fed}$ (Prop.~\ref{prop:marginal}) & 15 & $0.3 \pm 0.5$ & $0.071$ & $46.7\times$ \\
\midrule
Grid $G{=}50$ & 15 & $2.3 \pm 1.9$ & $0.116$ & $35.7\times$ \\
Grid $G{=}500$ & 15 & $3.3 \pm 1.9$ & $0.122$ & $3.0\times$ \\
\bottomrule
\end{tabular}
\end{table}

\section{Discussion}
\label{sec:discussion}

\textbf{Weighting, not communication, drives the failure.}
The contrast between weighted and unweighted scalar aggregation (Table~\ref{tab:main}) reveals that the aggregation weights silently determine \emph{whose} coverage is protected.
The mechanism on FeTS-2022 is calibration poverty: finite-sample correction pushes many small sites to the conservative boundary, and sample-size weighting suppresses their influence while equal-site weighting preserves it.
A $2\times$ stretch means the prediction set contains roughly twice the ground-truth tumor volume, far below the $83\times$ inflation of local CRC and plausibly reviewable in clinical workflows.
On FeTS-2022 the scalar unweighted average is competitive; risk-curve transmission becomes advantageous when sites need different points on the coverage-efficiency frontier, or when downstream losses require the full risk curve rather than a single threshold.
Designing federated calibration methods that explicitly separate target weights (whose risk is controlled) from borrowing weights (whose information is used) is an important direction for future work.

\textbf{Clinical significance.}
Our results show naive federated calibration exposes specific hospitals to site-specific false-negative errors, including one of our largest institutions, while appearing calibrated on average.
Vulnerability arises from case-mix and outcome heterogeneity, not merely from small sample size.

\textbf{Limitations.}
Proposition~\ref{prop:marginal} is conditional on pooled-threshold validity; a finite-sample CRC correction under partial exchangeability remains open.
Per-site thresholds are empirically validated but lack formal guarantees.
We use a single pre-trained model on one dataset; validation with additional anatomies, backbones, and models trained federatively is needed.
Adapting group-conditional FCP~\cite{gcfcp2026} to pixel-level CRC remains future work.

\section{Conclusion}
\label{sec:conclusion}

We quantified a failure mode of naive pooled federated calibration: on real multi-institutional brain tumor data, 40\% of hospitals exceed the target false-negative rate while the average looks fine.
This failure is driven by aggregation weights, not scalar communication: equal-site weighting improves reliability at comparable efficiency.
Risk-curve shrinkage provides a controllable family of operating points between local and pooled calibration, with $n_0$ validated via sensitivity analysis.
Direct budget optimization fails by exploiting vulnerable hospitals; the finite-sample correction is essential: removing it triples violations.
We encourage the community to evaluate calibration guarantees \emph{per site}, not just on average.

\bibliographystyle{splncs04}
\bibliography{references}

\end{document}